\documentclass{article}

\usepackage[preprint]{neurips_2025}

\usepackage[utf8]{inputenc} 
\usepackage[T1]{fontenc}    
\usepackage{hyperref}       
\usepackage{url}            
\usepackage{booktabs}       
\usepackage{amsfonts}       
\usepackage{nicefrac}       
\usepackage{microtype}      
\usepackage{xcolor}         
\usepackage{algorithm,algpseudocode}
\usepackage{mathrsfs}
\usepackage{graphicx}
\usepackage{mathtools}
\usepackage{amsmath}
\usepackage{amsthm}
\usepackage{xcolor}
\usepackage{amssymb}
\usepackage{soul}
\usepackage{wrapfig}    
\usepackage[normalem]{ulem}
\useunder{\uline}{\ul}{}

\renewcommand{\v}[1]{\boldsymbol{#1}}

\DeclareMathOperator*{\minimize}{min\,}
\DeclareMathOperator*{\maximize}{max\,}
\DeclareMathOperator*{\argmin}{argmin\,}
\DeclareMathOperator*{\subjectto}{subject\,\, to\,}
\DeclareMathOperator*{\diag}{Diag}

\newtheorem{theorem}{Theorem}

\newtheorem{corollary}{Corollary}
\newtheorem{lemma}{Lemma}
\newtheorem{remark}{Remark}
\newtheorem{assumption}{Assumption}

\newcommand{\R}{\mathbb{R}}
\newcommand{\reals}{\R}
\newcommand{\lt}{\left}
\newcommand{\rt}{\right}
\newcommand{\wt}{\widetilde}
\newcommand{\wh}{\widehat}
\newcommand{\wb}{\overline}
\newcommand{\bbeta}{\boldsymbol{\beta}}

\newcommand{\hbeta}{\widehat \bbeta}

\definecolor{asparagus}{rgb}{0.53, 0.66, 0.42}
\definecolor{atomictangerine}{rgb}{1.0, 0.6, 0.4}

\newcommand{\cC}{\mathcal{C}}
\newcommand{\cD}{\mathcal{D}}

\newcommand{\cS}{\mathcal{S}}

\def\boxit#1{\vbox{\hrule\hbox{\vrule\kern6pt
          \vbox{\kern6pt#1\kern6pt}\kern6pt\vrule}\hrule}}

\def\boxit#1{\vbox{\hrule\hbox{\vrule\kern6pt
          \vbox{\kern6pt#1\kern6pt}\kern6pt\vrule}\hrule}}

\usepackage{subcaption} 

\title{A Scalable Gradient-Based Optimization Framework for Sparse Minimum-Variance Portfolio Selection}

\author{%
  Sarat Moka\thanks{Corresponding author: \texttt{s.moka@unsw.edu.au}}\\
  School of Mathematics and Statistics\\
  University of New South Wales\\
  Sydney, NSW, Australia \\
\And
  Matias Quiroz\\
  School of Mathematical and Physical Sciences  \\
  University of Technology Sydney \\
  Sydney, NSW, Australia \\
\And
  Vali Asimit \\
  Bayes Business School \\ 
  City St George’s, University of London \\
  London, United Kingdom\\
\And
  Samuel Muller \\
  {Faculty of Science and Engineering} \\ 
  Macquarie University \\
  Sydney, NSW, Australia \\
}

\begin{document}

\maketitle

\begin{abstract}
Portfolio optimization involves selecting asset weights to minimize a risk-reward objective, 
such as the portfolio variance in the classical minimum-variance framework. 
Sparse portfolio selection extends this by imposing a cardinality constraint: only $k$ assets from a universe of $p$ may be included. 
The standard approach models this problem as a mixed-integer quadratic program and relies on commercial solvers to find the optimal solution. 
However, the computational costs of such methods increase exponentially with $k$ and $p$, making them too slow for problems of even moderate size. 
We propose a fast and scalable gradient-based approach that transforms the combinatorial sparse selection problem into a constrained continuous optimization task via Boolean relaxation, 
while preserving equivalence with the original problem on the set of binary points. 
Our algorithm employs a tunable parameter that transmutes the auxiliary objective from a convex to a concave function. 
This allows a stable convex starting point, followed by a controlled path toward a sparse binary solution as the tuning parameter increases and the objective moves toward concavity. 
In practice, our method matches commercial solvers in asset selection for most instances and, in rare instances, 
the solution differs by a few assets whilst showing a negligible error in portfolio variance.\\
\end{abstract}

\section{Introduction}
We propose a sparse portfolio selection framework that is computationally fast. To ground our ideas, we consider a specific instantiation of this framework---one designed for tractability and rigorous theoretical analysis. 
In particular, we consider the problem of minimizing portfolio variance across a universe of $p$ assets under a sparsity constraint that restricts investments to at most $k$ assets.
Minimum-variance portfolio is a cornerstone of portfolio theory, providing a principled risk-reduction strategy by focusing solely on return volatility. Originating in the seminal mean-variance framework of \citet{markowitz1952portfolio}, the minimum-variance portfolio remains widely adopted in practice \citep{clarke2011minimum}, especially when return forecasts are unreliable or difficult to estimate. Recent advances in portfolio selection—as discussed in Section~\ref{ref:related_work}---prioritize simplicity by enforcing sparsity (i.e., limiting non-zero asset weights). Sparsity enhances interpretability, mitigates estimation error, and improves practical feasibility \citep{hastie2015statistical, brodie2009sparse, fan2012vast, demiguel2009generalized}. In contrast, conventional portfolio optimization often yields dense allocations, where investing in assets with negligible weights inflates transaction costs and amplify sensitivity to covariance matrix estimation errors \citep{gao2013optimal}.

The sparse minimum-variance problem (see Section \ref{subsec:sparse_portfolio_selection}) is naturally framed as a mixed-integer quadratic program. Specialized commercial solvers---including branch-and-bound algorithms \citep{land1960} which guarantee globally optimal solutions given adequate computational time---exist for such problems. However, their runtimes scale exponentially with problem size and become infeasible in practical applications for moderately large $p$ and $k$.

We recast the sparse minimum-variance portfolio problem as a binary-constrained optimization over the binary $k$-cube $\{\v s \in \{0,1\}^p: \sum_{j = 1}^p s_j = k\}$. Drawing on the Boolean relaxation approach for best subset selection in regression by \citet{COMBSS2022}, we introduce a continuous auxiliary objective function defined on the simplex $\{ \v t \in [0, 1]^p: \sum_{j = 1}^p t_j = k\}$, which coincides with the reformulated function's values on the binary $k$-cube. Moreover, the auxiliary objective function has a positive tuning parameter that gradually transitions this auxiliary function from convex to concave as it increases. The convex phase facilitates stable initialization, while the concave phase ensures equivalence between the auxiliary function’s minima and those of the original problem. Building on these properties, we develop a variant of the {\em Frank-Wolfe} (conditional gradient) method that iteratively increases the parameter during optimization, thereby guiding the objective from convexity to concavity to reach the optimal solution of the binary constrained problem.

From a theoretical standpoint, as the tuning parameter increases continuously, our method is guaranteed to converge to an optimal solution of the original problem with binary constraints. In practical implementations, where the parameter is discretized, solutions may marginally deviate from optimality. Nonetheless, extensive numerical experiments on synthetic and real-world datasets demonstrate that such deviations are limited to a small number of assets, with minimal error in the objective value. Importantly, our algorithm eliminates combinatorial complexity entirely, enabling a computationally efficient and scalable solution.

The remainder of the paper is organized as follows. Section~\ref{sec:sps} formulates the sparse portfolio optimization problem and reviews existing approaches. Section~\ref{sec:bool-relax} introduces our Boolean relaxation framework, establishes its theoretical guarantees, and presents our algorithm. Section~\ref{sec:sims} demonstrates the performance of our method on synthetic and real-world financial datasets. Finally, Section~\ref{sec:conc} concludes the paper and outlines future research. All the theoretical results are proved in Appendix~\ref{app:proofs}.

\section{Sparse portfolio optimization}
\label{sec:sps}
In this section, we first introduce the notation employed in this paper and then present the formulation of the sparse portfolio selection problem. We use boldface notation for column vectors, for example, $\v u \in \reals^p$ denotes a $p$-dimensional real vector. We use $\v u^\top$ to denote the transpose of a vector $\v u$. The all-zeros vector, the all-ones vector, and the identity matrix are denoted by $\v 0$, $\v 1$, and $I$, respectively, and their dimensions are clear from the context.  Capital letters denote matrices, unless otherwise mentioned. For a binary vector $\v s \in \{0, 1\}^p$ that denotes which assets are selected or not, $|\v s|=\sum_j^p s_j$ denotes the number of selected assets. Moreover, for a $p\times p$-dimensional matrix $A$, $A_{[\v s]} \in \reals^{|\v s|\times |\v s|}$ is the square sub-matrix of $A$ obtained by removing all columns and rows of $A$ for which $s_j = 0$, $j=1,\ldots,p$. Similarly, the $|\v s|$-dimensional vector $\v u_{[\v s]}$ keeps the elements of the $p$-dimensional vector $\v u \in \reals^p$ corresponding to $s_j=1$. 
Throughout, two optimization problems are said to be equivalent if the solutions of one problem provides the solutions to the other, and vice versa.

\subsection{Minimum-variance portfolio optimization}
Let \(\v x_\tau = (x_{\tau, 1}, \dots, x_{\tau, p})^\top \in \reals^p\) denote the vector of returns for \(p\) assets at time \(\tau\). Consider a portfolio with weights \(\v \beta = (\beta_1, \dots, \beta_p)^\top\) such that the sum of the weights equals one, i.e., \(\v \beta^\top \v 1 = \sum_{j=1}^p \beta_j = 1\). The portfolio return at time \(\tau\) is \(\v \beta^\top \v x_\tau\). The mean return of the portfolio is given by \(\v \beta^\top \v \mu\), and its risk (volatility) is characterized by the variance \(\v \beta^\top \Sigma \v \beta\), where \(\v \mu\) represents the mean vector of asset returns and \(\Sigma\) is the covariance matrix of \(\v x_\tau\). 

The minimum-variance portfolio is an asset allocation strategy that constructs a portfolio with the lowest possible risk, measured as the variance of portfolio returns, regardless of expected returns. This strategy is ideal for risk-averse investors who prioritize minimizing risk over maximizing returns and is particularly useful when asset returns are difficult to forecast. The asset allocation for such an investor is given by the optimal weights obtained by solving
\begin{align}
\minimize_{\v \beta \in \reals^p} \v \beta^\top \Sigma \v \beta, \quad \subjectto \v \beta^\top \v 1 = 1.
\label{eqn:mvp}
\end{align}
We make the following key assumption throughout the paper.
\begin{assumption}\label{ass:main}
The covariance matrix $\Sigma$ is positive definite.
\end{assumption}
This assumption is not restrictive in practice because we can obtain $\Sigma$ using a positive definite covariance matrix estimator. Moreover, in practice, 
problem \eqref{eqn:mvp} is often augmented with a ridge regularization term by replacing the covariance matrix $\Sigma$ with $\Sigma + \lambda I$ for some $\lambda > 0$; see \cite{fastrich2015constructing, roncalli2013introduction}. In this case, Assumption~\ref{ass:main} holds trivially. 
The following lemma gives the optimal weights (i.e., the unique solution to \eqref{eqn:mvp}) under this assumption.
\begin{lemma}
\label{lem:key-res2}
The solution to the minimum-variance problem \eqref{eqn:mvp} is
\begin{align*}
\wh {\v \beta} = \frac{\Sigma^{-1} \v 1 }{\v 1^\top \Sigma^{-1} \v 1}.
\end{align*}
\end{lemma}

\subsection{Sparse portfolio selection}\label{subsec:sparse_portfolio_selection}
The minimum-variance optimization \eqref{lem:key-res2} yields dense portfolios, assigning nonzero weights to all assets. In many applications, some of the weights may be very small and, ideally, such assets should not be invested in to avoid transaction costs. Moreover, including only the relevant assets may reduce estimation errors. To mitigate these issues, {\em sparse portfolio selection} enforces an additional constraint on the number of assets selected, resulting in the optimization problem: For an integer $k \geq 1$,
\begin{align}
    \minimize_{\v \beta \in \reals^p} \v \beta^\top \Sigma \v \beta, \quad \subjectto \v 1^\top \v \beta = 1, \,\, \|\v \beta\|_0 \leq k,
    \label{eqn:opt-bss}
\end{align}
where $\|\v \beta\|_0$ denotes the number of non-zero elements in $\v \beta$. This NP-hard combinatorial problem makes naive exhaustive searches---evaluating all possible subsets of $k$ assets from $p$---computationally intractable even for moderately large $k$ and $p$ \citep{Natarajan1995, BKM16}.

A more efficient way to tackle problem \eqref{eqn:opt-bss} is via mixed-integer quadratic programming; for a comprehensive survey of its use in portfolio selection, refer to \cite{mencarelli2019complex}. A widely adopted technique in mixed-integer optimization is the {\em Big-M formulation}, which enforces logical conditions---such as whether asset $j$ is included ($s_j=1$) or excluded ($s_j=0$)---by embedding them directly into linear constraints. 
In particular, the Big-M formulation of problem \eqref{eqn:opt-bss} is  
\begin{align}
\begin{aligned}
    &\minimize_{\v \beta \in \reals^p, \,\, \v s \in \{0,1\}^p} \v \beta^\top \Sigma \v \beta,\\  
    &\hspace{4mm}\subjectto \v 1^\top \v \beta = 1, \,\,
    \v 1^\top \v s \leq k,\,\, \text{and}\\
    &\hspace{1.7cm} -M s_j \leq \beta_j \leq Ms_j, \,\,\, \forall j = 1, \dots, p,
\end{aligned}
\label{eqn:big-M}
\end{align}
where $M>0$ is such that at every optimal solution $\v \beta^*$ of \eqref{eqn:opt-bss}, $\max_j |\beta_j^*| \leq M$. This formulation ensures that $\beta_j = 0$ if and only if $s_j = 0$. 

Commercial solvers such as \texttt{CPLEX} \citep{cplex} and \texttt{Gurobi} \citep{gurobi} employ branch‐and‐bound algorithms \citep{land1960} to tackle formulation~\eqref{eqn:big-M}. Although they are guaranteed to converge to the global optimum given sufficient time, their computational cost grows exponentially, rendering them impractical for even moderately large~$k$ and~$p$.  

\subsection{Limitations and related work}
\label{ref:related_work}
As a first instantiation of our framework, we focus on a tractable setting that excludes additional weight constraints (e.g., short selling or minimum investments), as well as mean-variance objectives and risk-free assets. We leave these extensions for future work. Although we prove that our method recovers the exact sparse $k$ optimal solution as the tuning parameter increases continuously, this guarantee breaks down in practical implementations, where the tuning parameter is discretized and each step is only solved approximately using a gradient-based procedure. In contrast, certifiable optimization methods (e.g., the Big-M formulation) can guarantee optimality given enough computational budget. However, our method is computationally feasible for large-scale problems and achieves better solutions than the Big-M formulation under a fixed computational budget. 

Bertsimas and coauthors \citep{bertsimas2022scalable, bertsimas2021unified} present a major advance in scalable, certifiably optimal sparse portfolio selection under a mean-variance objective with weight constraints (e.g., minimum investments, no short selling). They incorporate a ridge penalty to improve tractability, and their cutting-plane algorithm outperforms the Big-M formulation for the ridge-regularized problem while converging faster with stronger regularization. However, this reliance on a regularization term makes their approach less applicable to our penalty-free setting. Other important works in cardinality-constrained minimization is surveyed in \cite{tillmann2024cardinality}.

\section{Methodology}
\label{sec:bool-relax}
To overcome the combinatorial bottleneck inherent in problem \eqref{eqn:opt-bss}, we introduce a Boolean relaxation that replaces the binary constraints with continuous ones, yielding an auxiliary objective function whose gradient and Hessian can be computed in closed form. This enables scalable continuous optimization.  

\subsection{Boolean relaxation}
The first step in our Boolean relaxation of the sparse portfolio selection problem \eqref{eqn:opt-bss} is to rewrite it as a binary constrained problem. 
It follows from Assumption \ref{ass:main} that 
each principal submatrix $\Sigma_{[\v s]}$ is invertible  for every $\v s \neq \v 0$ \citep{bhatia2009positive}. Consequently, \eqref{eqn:opt-bss} is equivalent to the binary constrained problem
\begin{align}
    \minimize_{\v s \in \{0, 1\}^p} \minimize_{\v \beta_{[\v s]} \in \reals^{|\v s|}} \v \beta_{[\v s]}^\top \Sigma_{[\v s]} \v \beta_{[\v s]}, \quad \subjectto \v 1^\top \v \beta = 1, \,\, \v 1^\top \v s \leq k,
    \label{eqn:opt-bc}
\end{align}
where in $\v \beta$ (without subscript), $\beta_j = 0$ for $s_j=0$. Using the notation $\Sigma_{[\v s]}^{-1} = (\Sigma_{[\v s]})^{-1}$, from Lemma~\ref{lem:key-res2}, 
$
\hbeta_{[\v s]} = \Sigma_{[\v s]}^{-1} \v 1/(\v 1^\top \Sigma_{[\v s]}^{-1} \v 1)
$ is the solution of the inner minimization in \eqref{eqn:opt-bc}. Thus, the problem can be rewritten as
\begin{align}
    \minimize_{\v s \in \{0, 1\}^p} \frac{1}{\v 1^\top  \Sigma_{[\v s]}^{-1} \v 1}, \quad \subjectto \v 1^\top \v s \leq k,
    \label{eqn:opt-bc4}
\end{align}
or, equivalently, 
\begin{align}
    \minimize_{\v s \in \{0, 1\}^p} - \v 1^\top  \Sigma_{[\v s]}^{-1} \v 1, \quad \subjectto \v 1^\top \v s \leq k,
    \label{eqn:opt-bc3}
\end{align}
where we take $\v 1^\top  \Sigma_{[\v s]}^{-1} \v 1$ to be zero when $\v s = \v 0$. Note that solving \eqref{eqn:opt-bc3}, which we refer to as the {\em target problem}, is equivalent to solving the {\em original problem} \eqref{eqn:opt-bss}. 
Theorem~\ref{thm:monotonicity} establishes the monotonicity of the optimal value of \eqref{eqn:opt-bc3} as a function of $k$. In particular, it implies that the inequalities in \eqref{eqn:opt-bc}, \eqref{eqn:opt-bc4} and \eqref{eqn:opt-bc3} can be replaced with equalities. In other words, inclusion of additional assets (increasing $k$) can only improve (lower) or retain the same variance of \eqref{eqn:opt-bss}. 
\begin{theorem}
\label{thm:monotonicity}
The optimal value of the target problem \eqref{eqn:opt-bc3}
is non-increasing in $k\in \{0, 1,\dots, p \}$. 
\end{theorem}

We now provide a Boolean relaxation of \eqref{eqn:opt-bc3} as an auxiliary continuous function on $[0, 1]^p$, controlled by a tuning parameter $\delta > 0$. 
To simplify the notation, define 
\begin{align}
T_{\v t} = \diag(\v t) \quad 
\text{and} \quad 
\wt \Sigma_{\v t} = T_{\v t}\Sigma T_{\v t} + \delta (I - T_{\v t}^2).
\label{eqn:defn-Sigma}
\end{align}
Then, our proposed Boolean relaxation of \eqref{eqn:opt-bc3} is given by
\begin{align}
    \label{eqn:opt-br2}
    \minimize_{\v t \in \cC_k} f_\delta(\v t), \quad \text{where}\,\,\,\,
f_\delta(\v t) = - \v t^\top {\wt \Sigma}_{\v t}^{-1} \v t,
\end{align}
and for each $k$ the constraint set $\cC_k$ is a polytope defined as
\begin{align}
\cC_k = \{ \v t \in [0, 1]^p : \v t^\top \v 1 \leq k\}.
\label{eqn:defn-polytope}
\end{align}

\subsection{Theoretical properties}
The following result, Theorem~\ref{thm:relaxation-properties}, shows why \eqref{eqn:opt-br2} is a relaxation of the target problem \eqref{eqn:opt-bc3}. It shows that $f_\delta(\v t)$ is continuous on the hypercube $[0, 1]^p$ and its shape can be controlled by the auxiliary parameter $\delta$ while keeping the values of $f_\delta(\v t)$ fixed---independent of $\delta$---at all the (binary) corners $\v s \in \{ 0,1\}^p$. In addition, (iii) shows that $f_\delta(\v t)$ increases with $\delta$ for any fixed interior point $\v t$---see Figure~\ref{fig:surface-plots}(a) for an illustration---while (iv) shows that the optimum of \eqref{eqn:opt-br2} is on a simplex.
\begin{theorem}
    \label{thm:relaxation-properties}
    The following hold:
    \begin{itemize}
        \item[(i)] The objective function $f_\delta(\v t)$ in \eqref{eqn:opt-br2} is continuous on $[0,1]^p$.
        \item[(ii)] For every binary vector $\v s \in \{0, 1\}^p$ (i.e., a corner point on the hypercube $[0, 1]^p$),
        \[
        f_\delta(\v s) = - \v 1^\top  \Sigma_{[\v s]}^{-1} \v 1, \quad \text{for all}\,\, \delta > 0.
        \]
        \item[(iii)] For every fixed $\v t \in (0, 1)^p$, $f_\delta(\v t)$ is monotonically increasing in $\delta > 0$.
        \item[(iv)] For any $k =1, \dots, p$ and $\delta > 0$, 
        \[
        \minimize_{\v t \in \cC_k} f_\delta(\v t) = \minimize_{\v t \in \cS_k} f_\delta(\v t),
        \]
        here, the simplex $\cS_k = \{\v t \in [0, 1]^p : \v t^\top \v 1 =k \}$ corresponds to the polytope $\cC_k$ given in \eqref{eqn:defn-polytope}.
    \end{itemize}
\end{theorem}

\begin{figure}[ht]
    \centering
    \begin{subfigure}[b]{0.57\textwidth}
    \raggedleft
    \includegraphics[width=1\textwidth, trim=50 70 10 110, clip]{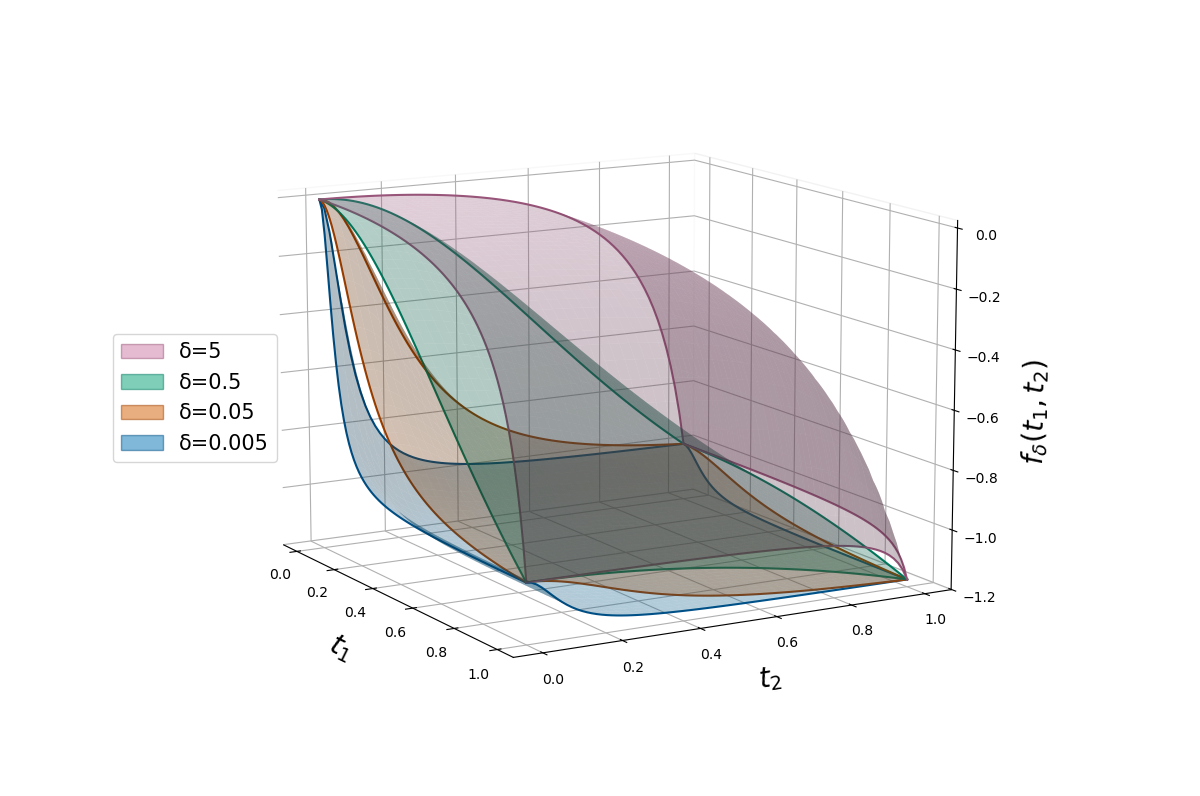}
    \caption{}
    \end{subfigure}
    \hspace{-0.1em} 
    \begin{subfigure}[b]{0.42\textwidth}
    \centering
    \includegraphics[width=1\textwidth, trim=5 0 0 20, clip]{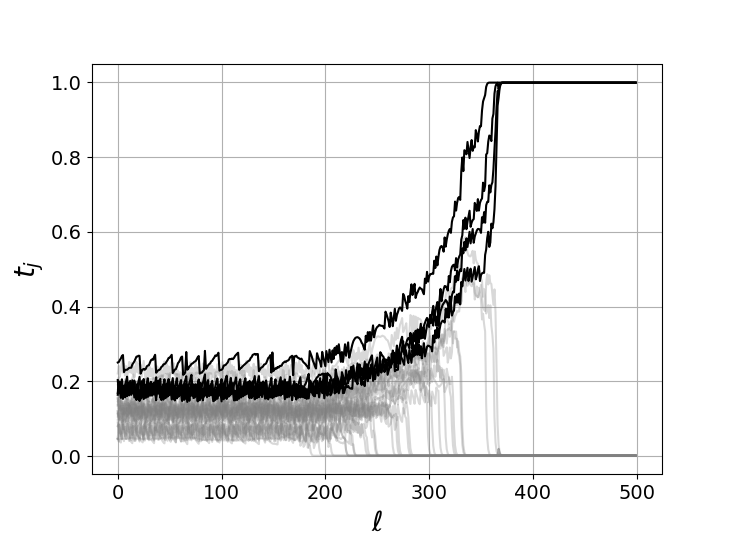}
    \caption{}
    \end{subfigure}
    \caption{\small Panel (a): Surface plots of the auxiliary objective function $f_\delta(\v t)$ for a $2\times 2$-dimensional covariance matrix $\Sigma$ for different values of $\delta$, displaying convexity for $\delta = 0.005$ and concavity for $\delta = 5$. The values of the function at the corners $\v s \in \{0, 1\}^2$ correspond to those of the discrete function $-\v 1^\top \Sigma_{[\v s]}^{-1}\v 1$ in \eqref{eqn:opt-bc3}. Panel (b): Iterative convergence of $t_j$'s toward $0$ or $1$ for a dataset of $p=31$ assets when $k = 4$ during the execution of Algorithm~\ref{alg:alg}. The paths of the four $t_j$ that converge to $1$ are shown in black, while other paths are shown in grey.}  
    \label{fig:surface-plots}
\end{figure}
Our next result, Theorem~\ref{thm:delta-prperties}, shows that  $f_\delta(\v t)$ exhibits convexity for sufficiently small values of $\delta$ and concavity for large ones. Using a $2 \times 2$ covariance matrix, Figure~\ref{fig:surface-plots}(a) illustrates the convexity of $f_\delta(\v t)$ for small values of $\delta$ and concavity for large values of $\delta$, and non-convexity and non-concavity in-between. The proof of this result is based on Lemma~\ref{lem:grad-hess} in Appendix~\ref{app:proofs}, which derives the Hessian of $f_\delta(\v t)$.
\begin{theorem}
    \label{thm:delta-prperties}
    Let $\eta_1$ and $\eta_p$ be the largest and smallest eigenvalues of $\Sigma$, respectively. Then, 
    \begin{itemize}
        \item[(i)] 
        $f_\delta(\v t)$ strictly concave over $[0,1]^p$ for $\delta \geq \eta_1$, and
        \item[(ii)] 
        for any $\varepsilon \in (0, 1)$, $f_\delta(\v t)$ is strictly convex over $[\varepsilon, 1]^p$ for $ \delta \leq 3\eta_p\varepsilon^2/(1 + 3\varepsilon^2)$.
    \end{itemize}
\end{theorem}

Strict concavity of $f_\delta(\v t)$ establishes the equivalence between the target problem \eqref{eqn:opt-bc3} and the Boolean relaxation~\eqref{eqn:opt-br2} for all $\delta \geq \eta_1$. Since \eqref{eqn:opt-bc3} is equivalent to the original problem \eqref{eqn:opt-bss}, the following result immediately follows from Theorem~\ref{thm:delta-prperties}(i).
\begin{corollary}
\label{cor:concavity-equi}
There exists $\delta_c \leq \eta_1$ such that for all $\delta \geq \delta_c$, the Boolean relaxation~\eqref{eqn:opt-br2} problem is equivalent to the original sparse portfolio selection problem~\eqref{eqn:opt-bss}.
\end{corollary}

Finally, Theorem~\ref{thm:sol-continuity} establishes the continuity of the solution to \eqref{eqn:opt-br2} as a function of the auxiliary variable~$\delta$. This is a consequence of the well-known optimization result popularly known as {\em Berge's maximum theorem}; refer to, e.g., \cite{sundaram1996first}.

\begin{theorem}
    \label{thm:sol-continuity}
    For each $k$, suppose $h_k^*$ and $\cD^*_k$ are the optimal value and optimal solution of the target problem \eqref{eqn:opt-bc3}. Further, let  
    \[
    h_k(\delta) = \minimize_{\v t \in \cC_k} f_\delta(\v t)\quad 
    \text{and}
    \quad 
    \cD_{k,\delta} = \argmin_{\v t \in \cC_k} f_\delta(\v t).
    \]
    Then, for any $\epsilon \in (0, \eta_1)$, $h_k(\delta)$ is continuous and $\cD_{k, \delta}$ is compact-valued and upper hemicontinuous in $\delta \in [\epsilon, \eta_1]$: for any sequence $\delta_\ell \to \eta_1$ and $\v t^{(\ell)} \in \cD_{k,\delta}$ with $\v t^{(\ell)} \to \v t^*$, we have $\v t^* \in \cD^*_k$. 
\end{theorem}
Since $\mathcal{D}_k^*$ constitutes an optimal solution to the original sparse portfolio problem,  Theorem~\ref{thm:sol-continuity} implies that by constructing a sequence of minimizers of  $f_\delta(\v t)$ along an increasing sequence of $\delta$ values approaching the largest eigenvalue $\eta_1$
of $\Sigma$, the limit of these minimizers will yield a solution to the original problem. This connection establishes a pathway to recover the sparse portfolio solution through a controlled tuning parameter $\delta$ converging to $\eta_1$, inspiring our algorithm below.

\subsection{Algorithm}
It is well-established that minimizing a strictly convex function over a convex set---such as the polytope $\cC_k$ or the simplex $\cS_k$---is a tractable convex optimization problem, guaranteeing convergence to a unique optimum and enabling $\epsilon$-approximate solutions via gradient-based methods in polynomial time~\citep{boyd2004convex}. In contrast, minimizing a strictly concave function over the same set is computationally hard, as all optima are confined to the vertices, making gradient descent highly sensitive to initialization and prone to suboptimal solutions. In our context, Theorem~\ref{thm:delta-prperties}(i) establishes that $f_\delta(\v t)$ is strictly concave on $\cC_k$ for any $\delta \geq \eta_1$, leading to an equivalence between the original sparse portfolio selection problem and its reformulation, but limiting the effectiveness of gradient-based methods due to the combinatorial landscape. Conversely, Theorem~\ref{thm:delta-prperties}(ii) shows that for sufficiently small positive $\delta$ values, $f_\delta(\v t)$ becomes strictly convex over a truncated hypercube $[\varepsilon, 1]^p$, for a given $\varepsilon > 0$, ensuring that gradient-based algorithms reliably converge to the unique optimal solution from any initialization within this region. Building on these properties, our approach begins with a small value of $\delta$ to obtain a stable initialization and gradually increases $\delta$ toward $\eta_1$, thereby transforming the objective from convex to concave and ultimately guiding the solution to the optimal point of the original binary-constrained problem.

Furthermore, Theorem~\ref{thm:relaxation-properties}(iv) establishes that the target optimization problem can be restricted to the simplex $\cS_k$ for each $k$. The conditional gradient method, also known as the {\em Frank--Wolfe} algorithm~\citep{frank1956algorithm, jaggi2013revisiting}, is particularly well-suited for optimization over simplices and enjoys strong convergence guarantees. Building on this, we introduce a modified variant, \texttt{Grid-FW}, which iteratively updates the vector $\v t$ while progressively increasing $\delta$ from a small initial value up to the largest eigenvalue $\eta_1$ of $\Sigma$. As mentioned earlier, this continuation strategy is motivated by Theorem~\ref{thm:sol-continuity}, which ensures that the optimal solution depends continuously on $\delta$. The full procedure is outlined in Algorithm~\ref{alg:alg}.
\begin{algorithm}
\caption{\texttt{Grid-FW}$(\Sigma, \alpha, \varepsilon, n, m)$}
\label{alg:alg}
\begin{algorithmic}[1]
\State Compute the largest and smallest eigenvalues $\eta_1$ and $\eta_p$ of $\Sigma$
\State Take $\delta_1 = 3\eta_p\varepsilon^2/(1 + 3\varepsilon^2)$ 
and $r = (\eta_1/\delta_1)^{1/(n-1)}$
\State Create a geometric grid $\{\delta_1, \delta_2, \dots, \delta_n \}$ where $\delta_\ell = \delta_1 r^{\ell-1}$, $\ell =1, \dots, n$
\State $\v t \leftarrow (k/p) \cdot \v 1$ and $\v s^* \leftarrow \v 0$
\For {$\ell = 1, 2, \dots, n$}
\For {$i = 1, \dots, m $}
\State Compute the gradient $\nabla f_{\delta_\ell}(\v t)$
\State Let $\v s \in \{0, 1\}^p$ with ones at the positions of the $k$ smallest components of $\nabla f_{\delta_{\ell}}(\v t)$
\State $\v t \leftarrow (1 - \alpha) \v t  + \alpha \v s$
\If {$\v 1^\top \Sigma_{[\v s^*]} \v 1 < \v 1^\top \Sigma_{[\v s]} \v 1 $ }
\State $\v s^* \leftarrow \v s$
\EndIf
\EndFor
\EndFor
\State \Return $\v s^*$, $\v s$ and $\v t$
\end{algorithmic}
\end{algorithm}

Algorithm~\ref{alg:alg}, \texttt{Grid-FW}$(\Sigma, \alpha, \varepsilon, n, m)$, executes for $n$ epochs and in each epoch, there are $m$ Frank-Wolfe steps. Using Theorem~\ref{thm:delta-prperties}, we select $\delta_1 = 3\eta_p\varepsilon^2/(1 + 3\varepsilon^2)$ for a small positive constant $\varepsilon \leq 0.1 (k/p)$, to make sure that the center $(k/p) \v 1$ of the simplex $\cS_k$ is within the truncated hypercube $[\varepsilon, 1]^p$. We create an increasing (geometric) sequence $\{\delta_1, \delta_2, \dots, \delta_n \}$ of $n$ values for the $\delta$ parameter with $\delta_\ell = \delta_1 r^{\ell-1}$ and $r = (\eta_1/\delta_1)^{1/n-1}$. Within each epoch $\ell$ (within the inner for-loop of Algorithm~\ref{alg:alg}), $\v t$ is iteratively updated $m$ times by solving  
\begin{align}
    &\v s \leftarrow \argmin_{\v u \in \cS_k} \v u^\top \nabla f_{\delta_\ell}(\v t), \quad \text{and then,}\quad  \v t \leftarrow (1 - \alpha) \v t  + \alpha \v s,\label{eqn:s-update}
\end{align}
where 
$\alpha \in (0, 1)$ is a fixed constant. Note that in our problem, it is easy to solve the minimization in~\eqref{eqn:s-update} because its solution $\v s$ must be a binary vector with ones correspond to the $k$ smallest values of the gradient $\nabla f_{\delta_\ell}(\v t)$; see line 8 of the algorithm. 

At $\delta_n = \eta_1$, the function $f_{\delta_n}(\v t)$ becomes strictly concave, so it attains its minimum at a corner of $\cS_k$, which is the {\em final} $\v s$ returned by line 15. For some datasets, we may slightly increase the performance by considering the {\em best} model $\v s^*$ among all models visited by the algorithm (see lines 10-12). Table~\ref{tab:illustration} presents the results for both the final model $\v s$ and the best model $\v s^*$ for two real-world datasets. Note that the final and the best models in each case achieve optimal or nearly optimal variance (just one false positive\footnote{Here, a false positive refers to our method selecting a suboptimal asset.} in two cases). We obtain the optimal models using the Big-M formulation in \texttt{CPLEX} (here $p$ and $k$ are sufficiently small for \texttt{CPLEX}  to confirm an optimal solution within a reasonable amount of time). 
Figure~\ref{fig:surface-plots}(b) shows the convergence of $\v t^{(\ell)}$ iteratively to a binary point for the dataset with $p=31$ when $k =4$. We see that $\v t^{(\ell)}$ converges much before the final epoch $n = 500$ as indicated by Corollary~\ref{cor:concavity-equi}. To take advantage of this, in our implementation of the algorithm, to reduce the running time, we also have an additional termination condition (not stated in the algorithm) that stops the algorithm when $\v t$ converges close to a corner point.

\begin{table}[h]
    \caption{Illustration of the {accuracy} of our method compared to optimal models on two datasets containing, respectively,  31 and  85 assets from \cite{chang2000heuristics} for $k$ up to 10. Here, {\bf FP} denotes the number of false positives and {\bf \% Err} denotes the percentage of error from our model compared to the optimal variance (i.e., 100 times the ratio of error to the optimal variance).
    \label{tab:illustration}}
    \centering
    \footnotesize
\begin{tabular}{r l c l c l c  l c l c l}
  \toprule
   \multicolumn{1}{c}{}
   & \multicolumn{5}{c}{{\bf Dataset with }$\v{p= 31}$} 
   & \multicolumn{1}{c}{}
    & \multicolumn{5}{c}{{\bf Dataset with }$\v{p=85}$} \\
  \cmidrule{2-6} \cmidrule{8-12}
    \multicolumn{2}{c}{}
   & \multicolumn{2}{c}{{\bf Final Model}} 
   & \multicolumn{2}{c}{{\bf Best Model}} 
    & \multicolumn{2}{c}{}
   & \multicolumn{2}{c}{{\bf Final Model}} 
   & \multicolumn{2}{c}{{\bf Best Model}} \\
  \cmidrule{3-4} \cmidrule{5-6}
   \cmidrule{9-10} \cmidrule{11-12}
  \textbf{$k$} 
    & \multicolumn{1}{l}{%
      \begin{tabular}{@{}c@{}}
        {\bf Optimal}\\
        {\bf Variance}
      \end{tabular}%
    }
    & \textbf{FP} & \textbf{\% Err}
    & \textbf{FP} & \textbf{\% Err}
    & \quad
    & \multicolumn{1}{l}{%
      \begin{tabular}{@{}c@{}}
        {\bf Optimal}\\
        {\bf Variance}
      \end{tabular}%
    }
    & \textbf{FP} & \textbf{\% Err}
    & \textbf{FP} & \textbf{\% Err}
  \\
  \midrule
  1  & 1.29$\times 10^{-3}$  & 0 & $0.0$  & 0 & $0.0$  
     &  & 4.50$\times 10^{-4}$  & 0 & $0.0$  & 0 & $0.0$  \\
  2  & 7.99$\times 10^{-4}$  & 0 & $0.0$   & 0 & $0.0$   
     &  & 2.74$\times 10^{-4}$  & 0 & $0.0$   & 0 & $0.0$   \\
  3  & 7.15$\times 10^{-4}$  & 0 & $0.0$   & 0 & $0.0$   
     &  & 2.19$\times 10^{-4}$  & 0 & $0.0$   & 0 & $0.0$   \\
  4  & 6.75$\times 10^{-4}$  & 0 & $0.0$   & 0 & $0.0$   
     &  & 1.97$\times 10^{-4}$  & 1 & $0.79$   & 1 & $0.79$   \\
  5  & 6.50$\times 10^{-4}$  & 0 & $0.0$   & 0 & $0.0$   
     &  & 1.84$\times 10^{-4}$  & 0 & $0.0$   & 0 & $0.0$   \\
  6  & 6.22$\times 10^{-4}$  & 0 & $0.0$   & 0 & $0.0$   
     &  & 1.72$\times 10^{-4}$  & 0 & $0.0$   & 0 & $0.0$   \\
  7  & 6.01$\times 10^{-4}$  & 0 & $0.0$   & 0 & $0.0$   
     &  & 1.64$\times 10^{-4}$  & 0 & $0.0$   & 0 & $0.0$   \\
  8  & 5.83$\times 10^{-4}$  & 1 & $0.26$  & 1 & $0.26$  
     &  & 1.56$\times 10^{-4}$  & 0 & $0.0$  & 0 & $0.0$  \\
  9  & 5.66$\times 10^{-4}$  & 0 & $0.0$   & 0 & $0.0$   
     &  & 1.52$\times 10^{-4}$  & 0 & $0.0$   & 0 & $0.0$   \\
  10 & 5.56$\times 10^{-4}$  & 0 & $0.0$   & 0 & $0.0$   
     & & 1.47$\times 10^{-4}$  & 0 & $0.0$   & 0 & $0.0$   \\
  \bottomrule
\end{tabular}
\end{table}

\begin{remark}(Scaling)
\normalfont
In real-world datasets, the eigenvalues of the covariance matrix $\Sigma$ can be quite small and span a wide range, as those for real datasets in  Table~\ref{tab:illustration} where, for $p=85$, the largest and smallest eigenvalues are about $0.024$ and $8.18 \times 10^{-5}$, respectively, and for $p=31$, they are $0.036$ and $2.26\times 10^{-4}$. Such small values can cause numerical instability when choosing~$\delta$. To address this, we modify the constraint $\v 1^\top \v \beta = 1$ and replace $\Sigma$ in \eqref{eqn:opt-bss} with its correlation matrix prior to Boolean relaxation. We incorporate this scaling in our algorithm; see Appendix~\ref{app:scaling} for details. Accordingly, $\eta_1$ and $\eta_p$ in Theorems~\ref{thm:delta-prperties} and \ref{thm:sol-continuity}, as well as in Algorithm~\ref{alg:alg}, are rescaled by $\min_j \Sigma_{j,j}$ and $\max_j \Sigma_{j,j}$. For instance, with $p=85$, the variances range from $4.49 \times 10^{-4}$ to $4.62 \times 10^{-3}$, so this scaling ensures the $\delta$-grid is numerically stable.
\end{remark}

\begin{remark}(Running time)
    \label{rem:running-time}
    \normalfont
    In Algorithm~\ref{alg:alg}, we perform $nm$ iterations; each iteration requires computing the gradient $\nabla f_\delta(\v t)$ at the current point $\v t$. From Lemma~\ref{lem:grad-hess} in Appendix~\ref{app:proofs}, the main cost in each gradient computation stems from solving a linear system $\Pi^{-1} \v 1$ for a positive definite matrix~$\Pi$. We do this via conjugate gradient, which takes $O(p^2)$ time within the required accuracy; refer to \cite{Golub1996} or \cite{saad2003iterative} for details on the complexity of conjugate gradient methods. Consequently, the total runtime of our method is $O(nmp^2)$.
\end{remark}

%
\section{Applications}
\label{sec:sims}
We benchmark our method \texttt{Grid-FW}, Algorithm~\ref{alg:alg}, against the Big-M formulation implemented using \texttt{CPLEX} for three different examples involving real as well as simulated datasets. We set $M=1$ for the Big-M formulation, following \cite{bertsimas2022scalable}. In Algorithm \ref{alg:alg}, we set $\alpha=0.05$, $\varepsilon = \min(0.1k/p,\, 0.001)$, $n=500$, and $m=10$. We compare the methods for fixed computational budgets and $k =  \lfloor z\% p \rfloor$, $z =10, 25, 50$. All experiments were run on a MacBook Pro with an Apple M4 Pro chip and 24 GB of memory, forcing single-threaded execution for numerical libraries (i.e., no parallel computations). Note that, because \texttt{CPLEX} is implemented in highly optimized low-level languages, our Python-based implementation is inherently at a disadvantage. Datasets, Python code for our method, and Julia code for \texttt{CPLEX} are provided in \url{https://github.com/saratmoka/grid-fw}.

In Table~\ref{tab:real_data_results}, the first computational budget $B$ (in seconds) is set equal to the time our method takes to solve each problem. To see if \texttt{CPLEX} can improve the solutions with more time, we also consider 60 and 300 seconds for \texttt{CPLEX}, where 300 seconds is the time limit imposed on \texttt{CPLEX} Big-M method by \cite{bertsimas2022scalable} for solving a similar problem. In Example~3, which involves a larger $p$, \texttt{CPLEX} is given 600 seconds. Across all examples, we compare the variance of the resulting portfolios under each computational budget. 

\subsection{Example 1: Three portfolio datasets}
For our first example, we use three of the five portfolio problems from \cite{chang2000heuristics}. Some results for the remaining two smallest datasets, with $p=31, 85$, are already illustrated in Table \ref{tab:illustration}. Here, we present results for the remaining three datasets with $p=89, 98, 225$.

\subsection{Example 2: S\&P 500}
Our second example involves a dataset from \cite{Asimit2025Risk} containing $6037$ daily returns of $p=441$ assets from firms that remained continuously listed in the S\&P 500 during the period January 2000 to December 2023. We split the dataset into three periods: 2000-2007, 2008-2015, and 2016-2023. Shrinkage is essential when estimating the covariance matrix for portfolio optimization \citep{Ledoit2004honey}, and we use the analytical non-linear shrinkage estimator in \cite{ledoit2020analytical}, which guarantees positive definiteness (recall Assumption \ref{ass:main}).

\subsection{Example 3: Assets with a low-rank plus noise structure}
Our final example consists of simulated data to compare the methods for very large values of $p$. A common approach (see e.g. \citealp{ross1976arbitrage}) to model the covariance matrix $\v \Sigma$ of the asset return vector $\v x_\tau$ is to assume the returns are driven by a linear combination $\Lambda \v \xi_\tau$ of uncorrelated lower dimensional factors $\v \xi_\tau = (\xi_{\tau,1}, \dots \xi_{\tau, q})^\top \in \mathbb{R}^q$, where $ \Lambda \in \mathbb{R}^{p \times q}$ are the factor loadings. We assume that
\[
\v x_\tau = \Lambda \v \xi_\tau + \v e, \,\, \v \xi_\tau \sim \mathcal{N}(\v 0, \mathrm{diag}(\nu^2_1, \dots, \nu^2_q)), \,\, \v e \sim \mathcal{N}(\v 0, d^2 I),\,\, \nu_1, \dots, \nu_q, d > 0.
\]
The resulting covariance is $\Sigma = \v \Lambda \mathrm{diag}(\nu_1, \dots, \nu_q)\v \Lambda^\top + d^2 I$. We divide the assets into two equally large groups, i.e.,\ $p/2$ assets in each. For the second group, the factor loadings for half of the factors (i.e.,\ $q/2$), are set to zero.  Moreover, the factors in the separate groups are normalized so that the relative variance of the portfolio based on assets from group 1 is roughly $10$ times larger than that of group 2. We set $q=1\%p$ and $d^2 = 0.05$, and $\nu^2_i$ simulated uniformly in  $[0.02^2, 0.05^2]$. 

\subsection{Results}

Table~\ref{tab:real_data_results} summarizes the results. Our method finished in at most 560 seconds when $p = 3000$, and it was often orders of magnitude faster across all examples. \texttt{CPLEX} timed out (over one day) on most problems, making it impossible to verify optimality. For small-scale problems, our variance reduction over \texttt{CPLEX} with equal run time was modest with only a few percentage points. When \texttt{CPLEX} is allowed more time for small-scales examples, it often finds a better optimum than our method. By contrast, for large-scale problems, our method delivered significantly lower variance than \texttt{CPLEX} with the same budget, and even extra time for \texttt{CPLEX} did not guarantee closing the gap.

\renewcommand{\arraystretch}{1.2}
\begin{table}[ht]
    \caption{Results for the three examples in Section \ref{sec:sims}. \textbf{Grid-FW} denotes our method, and the corresponding columns contain the time $B$ to run our method with $n = 500$ and $m = 10$. For \textbf{CPLEX Big-M}, \textbf{RV}$_B$ denotes the relative variance difference to \textbf{Grid-FW} obtained after $B$ seconds. \textbf{RV}$_\tau$ denotes the same quantity obtained after $\tau$ seconds. \textbf{RV}$_{\tau}$ is NA whenever $B>\tau$.  A positive value indicates a larger optimal variance (worse solution) for \texttt{CPLEX}. The $^*$ symbol denotes a difference of exactly zero. 
    \label{tab:real_data_results}}
\centering
\footnotesize
\begin{tabular}{lrrrrcrrr}
\toprule
                                       &                         &                         & \multicolumn{2}{c}{\textbf{Grid-FW}}                                     & \multicolumn{1}{l}{} & \multicolumn{3}{c}{\textbf{CPLEX Big-M}}                                                                                \\ \cline{4-5} \cline{7-9} 
\multicolumn{1}{c}{\textbf{Example 1}} & \multicolumn{1}{c}{$p$} & \multicolumn{1}{c}{$k$} & \multicolumn{1}{c}{\textbf{Time $B$}} & \multicolumn{1}{c}{\textbf{Var}} &                      & \multicolumn{1}{c}{\textbf{RV$_B$}} & \multicolumn{1}{c}{\textbf{RV$_{60}$}}  & \multicolumn{1}{c}{\textbf{RV$_{300}$}} \\ \cline{1-5} \cline{7-9} 
{\ul }                                 & $89$                    & $8$                     & 0.5                                   & 2.11$\times 10^{-4}$             &                      & 0\%$^*$             & 0\%$^*$                 & 0\%$^*$                 \\
                                       &                         & $22$                    & 0.4                                   & 1.61$\times 10^{-4}$             &                      & 0.4\%\phantom{$^\star$}                             & -0.1\%\phantom{$^\star$}                                & -0.1\%\phantom{$^\star$}                                \\
                                       &                         & $44$                    & 0.4                                   & 1.39$\times 10^{-4}$             &                      & 1.34\%\phantom{$^\star$}                            & 0\%$^*$                 & 0\%$^*$                 \\
                                       & $98$                    & $9$                     & 0.7                                   & 1.30$\times 10^{-4}$             &                      & 4.6\%\phantom{$^\star$}                             & 0\%$^*$                 & 0\%$^*$                 \\
                                       &                         & $24$                    & 0.6                                   & 9.52$\times 10^{-5}$             &                      & 2.5\%\phantom{$^\star$}                             & 0\%$^*$                 & 0\%$^*$                 \\
                                       &                         & $49$                    & 0.6                                   & 8.18$\times 10^{-5}$             &                      & 0.6\%\phantom{$^\star$}                             & -0.1\%\phantom{$^\star$}                                & -0.1\%\phantom{$^\star$}                                \\
                                       & $225$                   & $22$                    & 1.5                                   & 1.42$\times 10^{-4}$             &                      & 6.64\%\phantom{$^\star$}                            & -2.7\%\phantom{$^\star$}                                & -2.5\%\phantom{$^\star$}                                \\
                                       &                         & $56$                    & 1.6                                   & 7.81$\times 10^{-5}$             &                      & 31.3\%\phantom{$^\star$}                            & -0.7\%\phantom{$^\star$}                                & -1.7\%\phantom{$^\star$}                                \\
                                       &                         & $112$                   & 1.5                                   & 4.54$\times 10^{-5}$             &                      & 37.7\%\phantom{$^\star$}                            & 3.1\%\phantom{$^\star$}                                 & 2.6\%\phantom{$^\star$}                                 \\ \hline
\multicolumn{1}{c}{\textbf{Example 2}} & \multicolumn{1}{c}{$p$} & \multicolumn{1}{c}{$k$} & \multicolumn{1}{c}{\textbf{Time $B$}} & \multicolumn{1}{c}{\textbf{Var}} &                      & \multicolumn{1}{c}{\textbf{RV$_B$}} & \multicolumn{1}{c}{\textbf{RV$_{60}$}}  & \multicolumn{1}{c}{\textbf{RV$_{300}$}} \\ \cline{1-5} \cline{7-9} 
2000-2007                              & $441$                   & $44$                    & 3.9                                   & 2.95$\times 10^{-5}$             &                      & 21.1\%\phantom{$^\star$}                            & 2.6\%\phantom{$^\star$}                                 & 0.6\%\phantom{$^\star$}                                 \\
                                       &                         & $110$                   & 3.8                                   & 2.31$\times 10^{-5}$             &                      & 19.5\%\phantom{$^\star$}                            & 3.7\%\phantom{$^\star$}                                 & 1.8\%\phantom{$^\star$}                                 \\
                                       &                         & $220$                   & 3.8                                   & 2.00$\times 10^{-5}$             &                      & 32.3\%\phantom{$^\star$}                            & 6.9\%\phantom{$^\star$}                                 & 0.9\%\phantom{$^\star$}                                 \\
2008-2015                              & $441$                   & $44$                    & 3.7                                   & 3.52$\times 10^{-5}$             &                      & 19.5\%\phantom{$^\star$}                            & 3.7\%\phantom{$^\star$}                                 & 1.8\%\phantom{$^\star$}                                 \\
                                       &                         & $110$                   & 3.6                                   & 2.76$\times 10^{-5}$             &                      & 15.7\%\phantom{$^\star$}                            & 4.7\%\phantom{$^\star$}                                 & 1.7\%\phantom{$^\star$}                                 \\
                                       &                         & $220$                   & 3.6                                   & 2.43$\times 10^{-5}$             &                      & 39.8\%\phantom{$^\star$}                            & 6.2\%\phantom{$^\star$}                                 & 0.7\%\phantom{$^\star$}                                 \\
2016-2023                              & $441$                   & $44$                    & 3.4                                   & 4.27$\times 10^{-5}$             &                      & 24.1\%\phantom{$^\star$}                            & 5.5\%\phantom{$^\star$}                                 & 1.0\%\phantom{$^\star$}                                 \\
                                       &                         & $110$                   & 3.4                                   & 3.42$\times 10^{-5}$             &                      & 23.8\%\phantom{$^\star$}                            & 4.9\%\phantom{$^\star$}                                 & 2.0\%\phantom{$^\star$}                                 \\
                                       &                         & $220$                   & 3.4                                   & 2.99$\times 10^{-5}$             &                      & 35.6\%\phantom{$^\star$}                            & 4.8\%\phantom{$^\star$}                                 & 0.7\%\phantom{$^\star$}                                 \\ \hline
\multicolumn{1}{c}{\textbf{Example 3}} & \multicolumn{1}{c}{$p$} & \multicolumn{1}{c}{$k$} & \multicolumn{1}{c}{\textbf{Time $B$}} & \multicolumn{1}{c}{\textbf{Var}} &                      & \multicolumn{1}{c}{\textbf{RV$_B$}} & \multicolumn{1}{c}{\textbf{RV$_{300}$}} & \multicolumn{1}{c}{\textbf{RV$_{600}$}} \\ \cline{1-5} \cline{7-9} 
                                       & $1000$                  & $100$                   & 13.6                                  & 5.01$\times 10^{-4}$             &                      & 61.3\%\phantom{$^\star$}                            & -0.1\%\phantom{$^\star$}                                  & -0.1\%\phantom{$^\star$}                                \\
                                       &                         & $250$                   & 16.0                                  & 2.00$\times 10^{-4}$             &                      & 304.3\%\phantom{$^\star$}                           & 0.0\%\phantom{$^\star$}                                 & -0.0\%\phantom{$^\star$}                                \\
                                       &                         & $500$                   & 20.6                                  & 1.00$\times 10^{-4}$             &                      & 708.8\%\phantom{$^\star$}                           & 0.0\%\phantom{$^\star$}                                 & -0.0\%\phantom{$^\star$}                                \\
                                       & $2000$                  & $200$                   & 155.0                                 & 2.51$\times 10^{-4}$             &                      & 47.4\%\phantom{$^\star$}                            & 47.4\%\phantom{$^\star$}                                & 0.9\%\phantom{$^\star$}                                 \\
                                       &                         & $500$                   & 168.2                                 & 1.00$\times 10^{-4}$             &                      & 269.6\%\phantom{$^\star$}                           & 30.4\%\phantom{$^\star$}                                & 34.4\%\phantom{$^\star$}                                \\
                                       &                         & $1000$                  & 107.1                                 & 5.00$\times 10^{-5}$             &                      & 7.6\%\phantom{$^\star$}                             & 7.6\%\phantom{$^\star$}                                 & 1.6\%\phantom{$^\star$}                                 \\
                                       & $3000$                  & $300$                   & 558.7                                 & 1.68$\times 10^{-4}$             &                      & 28.1\%\phantom{$^\star$}                            & NA\,\,\,                                      & 18.2\%\phantom{$^\star$}                                \\
                                       &                         & $750$                   & 548.7                                 & 6.67$\times 10^{-5}$             &                      & 61.5\%\phantom{$^\star$}                            & NA\,\,\,                                      & 61.5\%\phantom{$^\star$}                                \\
                                       &                         & $1500$                  & 376.5                                 & 3.33$\times 10^{-5}$             &                      & 6.9\%\phantom{$^\star$}                            & NA\,\,\,                                      & 6.9\%\phantom{$^\star$} \\ 
\bottomrule
\end{tabular}
\end{table}
\section{Conclusion and future research}
\label{sec:conc}
We proposed a novel framework for sparse portfolio selection that employs Boolean relaxation to avoid intractable combinatorial computations. We devised an auxiliary objective function with a tuning parameter that transmutes the function from convex to concave. We proved that this transition helps in converging to the optimum of the original problem. Our method---suitable with or without ridge regularization---offers a practical and scalable alternative to the Big-M formulation in \texttt{CPLEX} for sparse portfolio selection.

Future research will consider more general objectives, such as the mean-variance framework, and additional constraints such as no short-selling or minimum and maximum investment limits. We believe such extensions can be competitive alternatives to \cite{bertsimas2022scalable}. 





\appendix
\section{Scaling}
\label{app:scaling}
To explain our scaling approach, let $R$ be the correlation matrix obtained from the covaraince matrix $\Sigma$, that is, 
\(
R_{i, j} = \Sigma_{i, j}/\sqrt{\Sigma_{i, i} \Sigma_{j, j}}, \)
\(
i, j \in \{1, \dots, p\}.
\)
Let $\v w = (1/\sqrt{\Sigma_{1, 1}}, \dots, 1/\sqrt{\Sigma_{p, p}})$. 
Then, the binary constrained problem \eqref{eqn:opt-bc3} can be rewritten as
\begin{align}
    \minimize_{\v s \in \{0, 1\}^p} - \v w_{[\v s]}^\top  R_{[\v s]}^{-1} \v w_{[\v s]}, \quad \subjectto |\v s| \leq k.
    \label{eqn:opt-bc3-mod}
\end{align}
With
\(
\wt R_{\v t} = T_{\v t} R T_{\v t} + \delta (I - T_{\v t}^2)
\)
and
\(
\wh f_{\delta}(\v t) = - (\v w \odot \v t)^\top \wt R_{\v t}^{-1}
(\v w \odot \v t),
\)
$\v t \in [0, 1]^p$,
we can consider a Boolean relaxation of \eqref{eqn:opt-bc3-mod} as
\begin{align*}
    \minimize_{\v t \in \cC_k} \wh f_{\delta}(\v t),
\end{align*}
where $\cC_k$ is the polytope defined by \eqref{eqn:defn-polytope}.
In particular, at the interior points $\v t \in (0,1)^p$, we can simplify the objective function to be
\[
\wh f_{\delta}(\v t) = - \v 1^\top \wh \Pi_{\v t}^{-1}
\v 1, \quad \text{where}\,\,\,\, 
\wh \Pi_{\v t} = \Sigma + \delta D_{\v t} \diag(\v v),
\]
with $\v v = \v 1/(\v w\odot \v w)$, which is the diagonal of $\Sigma$ (i.e., variances of the assets). See Lemma~\ref{lem:grad-hess} in Appendix~\ref{app:proofs} for expressions of the gradient 
and the Hessian of $\wh f_{\delta}(\v t)$ at each $\v t \in (0,1)^p$. 
Furthermore, the proof of Theorem~\ref{thm:delta-prperties} in Appendix~\ref{app:proofs} establishes that 
\begin{itemize}
    \item[(i)] $\wh f_\delta(\v t)$ strictly concave over $[0,1]^p$ for $\delta \geq \eta_1/\min_j v_j$, and 
    \item[(ii)] for every $\varepsilon \in (0, 1)$, it is strictly convex over $[\varepsilon, 1]^p$ for $\delta  \leq \frac{\eta_p}{\max_j v_j}\lt(\frac{3  \varepsilon^2}{1 + 3  \varepsilon^2}\rt)$.
\end{itemize}
Here, recall that $\eta_1$ and $\eta_p$ are the largest and smallest eigenvalues of $\Sigma$, respectively. 


\section{Proofs}
\label{app:proofs}

\begin{proof}[Proof of Lemma~\ref{lem:key-res2}]
The problem is convex and has a unique minimum. Using the method of Lagrange multipliers gives the result.
\end{proof}

\begin{proof}[Proof of Theorem~\ref{thm:monotonicity}]
It is well-known that every principal submatrix of a symmetric positive definite matrix is also positive definite \citep{Golub1996}. Since $\Sigma$ is symmetric positive definite, this implies that $\Sigma_{[\v s]}$ is also positive definite for every binary vector $\v s \neq \v 0$, and hence 
    \begin{align*}
            h(k) \coloneqq  \maximize_{\v s \in \{0, 1\}^p, \, \, \v 1^\top \v s = k} \v 1^\top \Sigma_{[\v s]}^{-1} \v 1.
    \end{align*}
is well-defined for all $k = 1, \dots, p$.

Now fix $k$ and consider $\v s \in \{0, 1\}^p$ such that $\v 1^\top \v s = k$. Consider $\v s'$ with $\v 1^\top \v s' = k+1$ such that the Hamming distance between $\v s$ and $\v s'$ is $1$, that is, there is a unique $j \in \{1, \dots, p\}$ such that $s_i = s'_i$ for all $i \neq j$ and $0 = s_j < s_j' =1$. Then, the proof is complete if we show that
\begin{align}
    \v 1^\top  \Sigma_{[\v s]}^{-1} \v 1 \leq \v 1^\top  \Sigma_{[\v s']}^{-1} \v 1,
\label{eqn:monotonicity}
\end{align}
because, under \eqref{eqn:monotonicity},
$\v 1^\top  \Sigma_{[\v s]}^{-1} \v 1 \leq h(k+1)$ and hence $h(k) \leq h(k+1)$ by maximizing the left hand size over all $\v s$ with $|\v s| = k$.

We prove \eqref{eqn:monotonicity} using the popular Banachiewicz inversion lemma \citep{YY05}. 
Without loss of generality we assume that such $\v s$ and $\v s'$ differ at the index $p$. 
Then, we write 
\[
\Sigma_{[\v s']} 
= 
\begin{bmatrix}
    \Sigma_{[\v s]} & & \v a \\
    & & \\
    \v a^\top & & b 
\end{bmatrix},
\]
for some appropriate vector $\v a$ and a scalar $b$. Note that both $\Sigma_{[\v s']}$ and $\Sigma_{[\v s]}$ are invertible and $b > 0$.  
Therefore, the Schur complement $c = b - \v a^\top \Sigma_{[\v s]}^\top \v a$ must be non-zero; see \cite{YY05}. 
Then,
\begin{align*}
    \Sigma_{[\v s']}^{-1} 
    =
    \begin{bmatrix}
        I & & - \frac{1}{c} \Sigma_{[\v s]}^{-1} \v a \\
        & & \\
        \v 0^\top & &\frac{1}{c}
    \end{bmatrix}
    \begin{bmatrix}
         \Sigma_{[\v s]}^{-1} & & \v 0 \\
         & & \\
        - \v a^\top  \Sigma_{[\v s]}^{-1} & & 1
    \end{bmatrix}
    =
    \begin{bmatrix}
         \Sigma_{[\v s]}^{-1} + \frac{1}{c} \Sigma_{[\v s]}^{-1} \v a \v a^\top  \Sigma_{[\v s]}^{-1} & & - \frac{1}{c} \Sigma_{[\v s]}^{-1} \v a \\
         & & \\
        - \frac{1}{c} \v a^\top  \Sigma_{[\v s]}^{-1} & & \frac{1}{c}
    \end{bmatrix}.
\end{align*}
Note that for any matrix $A$, $\v 1^\top A \v 1$ is simply the sum of all the elements of $A$. Similarly, for any vector $\v u$, $\v 1^\top \v u$ is the sum of all the elements of $\v u$. Thus, 
\begin{align*}
    \v 1^\top \Sigma_{[\v s']}^{-1} \v 1
    &= \v 1^\top \Sigma_{[\v s]}^{-1} \v 1 + \frac{1}{c} (\v 1^\top \Sigma_{[\v s]}^{-1} \v a)^2 - 2 \frac{1}{c}\v 1^\top \Sigma_{[\v s]}^{-1} \v a + \frac{1}{c} \\
    &= \v 1^\top \Sigma_{[\v s]}^{-1} \v 1 + \frac{1}{c} \lt(\v 1^\top \Sigma_{[\v s]}^{-1} \v a - 1\rt)^2.
\end{align*}
Therefore, to establish \eqref{eqn:monotonicity}, we just need to show that $c > 0$. For this, using simple block-matrix multiplication, observe that 
\begin{align*}
    \begin{bmatrix}
         I & & \v 0 \\
         & & \\
        - \v a^\top  \Sigma_{[\v s]}^{-1} & & 1
    \end{bmatrix}
    \begin{bmatrix}
        \Sigma_{[\v s]} & & \v a \\
        & & \\
        \v a^\top & & b \\
    \end{bmatrix} 
    \begin{bmatrix}
         I & & - \Sigma_{[\v s]}^{-1} \v a \\
         & & \\
         \v 0 & & 1
    \end{bmatrix}
    =     
    \begin{bmatrix}
         \Sigma_{[\v s]} & & \v 0 \\
         & & \\
         \v 0^\top & & c
    \end{bmatrix}.
\end{align*}
Since $\Sigma_{[\v s']}$ is positive definite, for any scalar $v \neq 0$ with $\v u = [\v 0^\top, v]^\top$, using the left hand side expression of the above equality,
\begin{align*}
    \v u^\top \begin{bmatrix}
         I &  \v 0 \\
         &  \\
        - \v a^\top  \Sigma_{[\v s]}^{-1} &  1
    \end{bmatrix}
    \begin{bmatrix}
        \Sigma_{[\v s]} & & \v a \\
        & & \\
        \v a^\top & & b \\
    \end{bmatrix} 
    \begin{bmatrix}
         I &  - \Sigma_{[\v s]}^{-1} \v a \\
         &  \\
         \v 0 &  1
    \end{bmatrix}
    \v u
    =    
\begin{bmatrix}
        - v \v a^\top  \Sigma_{[\v s]}^{-1} &  v
    \end{bmatrix}
        \Sigma_{[\v s']}
    \begin{bmatrix}
      - v \Sigma_{[\v s]}^{-1} \v a \\
         \\
         v 
    \end{bmatrix}
    > 0,
\end{align*}
and hence,
\[
\v u^\top     
\begin{bmatrix}
     \Sigma_{[\v s]} & & \v 0 \\
     & & \\
     \v 0^\top & & c
\end{bmatrix}
\v u = c v^2 > 0,
\]
which is possible only if $c > 0$.
\end{proof}

We need the following lemma in proving Theorem~\ref{thm:relaxation-properties}(iii).
\begin{lemma}
    \label{lem:simplified-f}
For every interior point $\v t \in (0, 1)^p$, the objective function $f_\delta(\v t)$ can be expressed as 
    \begin{align*}
    f_\delta(\v t) = - \v 1^\top \Pi_{\v t}^{-1} \v 1,\quad \text{with}\,\,\,\,
    \Pi_{\v t} = \Sigma + \delta D_{\v t} \,\,\, \text{and}\,\,\, D_{\v t} = T_{\v t}^{-2} - I.
    \end{align*}
\end{lemma}
\begin{proof}
    Note that $\v t^\top = \v 1^\top T_t$. Then,
\begin{align*}
f_\delta(\v t) & = -\v 1^\top T_{\v t} \widetilde{\Sigma}_{\v t}^{-1} T_{\v t} \v 1 \\
& = -\v 1^\top \left(T^{-1}_{\v t} \widetilde{\Sigma}_{\v t} T_{\v t}^{-1}\right)^{-1} \v 1 
\\
& = -\v 1^\top \left(\Sigma + \delta(T^{-2}_{\v t} - I)\right)^{-1} \v 1\\
& = -\v 1^\top \Pi^{-1}_{\v t} \v 1,
\end{align*}
since $T_{\v t}$ is invertible for every interior point $\v t$.
\end{proof}

\begin{proof}[Proof of Theorem~\ref{thm:relaxation-properties}]
To establish the continuity in (i), let $\v t_{1}, \v t_{2}, \cdots \in [0, 1]^p$ be a sequence of points converging to $\v t$ 
(element-wise or in the Euclidean norm). If $\v t$ contains zeros, without loss of generality, assume all these zeros appear at the front. Then, we can write
\[
T_{\v t_\ell} = 
\begin{bmatrix}
    F'_{\v t_\ell} & \v 0 \\
    \v 0 & F''_{\v t_\ell}
\end{bmatrix},
\]
where $F'_{\v t_\ell}$ is a diagonal matrix of the front elements of $\v t_\ell$ (that converge to zero) and $F''_{\v t_\ell}$ is a diagonal matrix of the remaining elements. Therefore, for all sufficiently large $\ell$, we can assume that $F''_{\v t_\ell}$ is invertible. 
Recall from \eqref{eqn:defn-Sigma} that $ \wt \Sigma_{\v t} = T_{\v t} \Sigma T_{\v t} + \delta (I - T_{\v t}^2)$. By dividing the $\Sigma$ matrix into four suitable sub-matrices, we can write
\begin{align*}
    \wt \Sigma_{\v t_\ell} &= 
    \begin{bmatrix}
    F'_{\v t_\ell} & \v 0 \\
    \v 0 & F''_{\v t_\ell}
    \end{bmatrix}
    \begin{bmatrix}
    \Sigma' & \wb \Sigma\\
    \wb \Sigma^\top & \Sigma''
    \end{bmatrix}
    \begin{bmatrix}
    F'_{\v t_\ell} & \v 0 \\
    \v 0 & F''_{\v t_\ell}
    \end{bmatrix}
    + 
    \delta  
    \begin{bmatrix}
    I - (F'_{\v t_\ell})^2 & \v 0 \\
    \v 0 & I - (F''_{\v t_\ell})^2
    \end{bmatrix}\\
    &=
    \begin{bmatrix}
    F'_{\v t_\ell} \Sigma' F'_{\v t_\ell} & F'_{\v t_\ell} \wb \Sigma F''_{\v t_\ell}\\
    F''_{\v t_\ell} \wb \Sigma^\top F'_{\v t_\ell} & F''_{\v t_\ell} \Sigma'' F''_{\v t_\ell}
    \end{bmatrix}
    + 
    \delta  
    \begin{bmatrix}
    I - (F'_{\v t_\ell})^2 & \v 0 \\
    \v 0 & I - (F''_{\v t_\ell})^2
    \end{bmatrix}\\
    &=
    \underbrace{\begin{bmatrix}
    \delta\lt(I - (F'_{\v t})^2\rt) & \v 0\\
    \v 0 & F''_{\v t} \Sigma'' F''_{\v t} + \delta\lt(I - (F''_{\v t})^2\rt)
    \end{bmatrix}}_{\wt \Sigma_{\v t} }\\
    &\hspace{1.5cm}+ 
    \underbrace{\begin{bmatrix}
    F'_{\v t_\ell} \Sigma' F'_{\v t_\ell} & F'_{\v t_\ell} \wb \Sigma F''_{\v t_\ell}\\
    F''_{\v t_\ell} \wb \Sigma^\top F'_{\v t_\ell} & F''_{\v t_\ell} \Sigma'' F''_{\v t_\ell} - F''_{\v t} \Sigma'' F''_{\v t} 
    \end{bmatrix}
    + 
    \delta  
    \begin{bmatrix}
    (F'_{\v t})^2 - (F'_{\v t_\ell})^2 & \v 0 \\
    \v 0 & (F''_{\v t})^2 - (F''_{\v t_\ell})^2
    \end{bmatrix}}_{E_\ell}\\
    &= 
    \wt \Sigma_{\v t} + E_{\ell},
\end{align*}
where $E_\ell = \wt \Sigma_{\v t_\ell} - \wt \Sigma_{\v t}$ clearly converges to an all-zero matrix. 
That means, the spectral radius of $E_\ell$ is less than $1$ for all sufficiently large values of $\ell$. Thus using the Neumann series \citep{horn2012matrix}, we can show that 
\[
\wt \Sigma_{\v t_\ell}^{-1} = (\wt \Sigma_{\v t} + E_\ell)^{-1} \rightarrow \wt \Sigma_{\v t}^{-1}, \quad \text{as}\, \, \ell \to \infty.
\]
As a consequence, $\lim_{\ell \to \infty} f_\delta(\v t_\ell) = f_\delta(\v t)$ for any $\delta > 0$.

To establish (ii),  observe that for any binary vector $\v s \in \{0,1\}^p$,
\[
\big(\wt \Sigma_{\v s}\big)_{[\v s]} = \big(T_{\v s} \Sigma T_{\v s}\big)_{[\v s]} = \Sigma_{[\v s]}, \quad \text{and}\quad 
\big(\wt \Sigma_{\v s}\big)_{[\v 1 - \v s]} = \delta I.
\]
Thus, $f_\delta(\v s) = - \v 1^\top  \Sigma_{[\v s]}^{-1} \v 1$.

To establish (iii), observe from Lemma~\ref{lem:simplified-f} that 
\begin{align*}
    \frac{d}{d\delta} f_\delta(\v t) & = -\v 1^\top \frac{d}{d\delta}\Pi^{-1}_{\v t} 1^\top \\  
    & = \v 1 ^\top \Pi_{\v t}^{-1}\left(\frac{d}{d\delta}\Pi_{\v t} \right)\Pi_{\v t}^{-1}\v 1 \\
    & = \v u^T \left(\frac{d}{d\delta}\Pi_{\v t} \right) \v u \\
    & = \v u^T \left(T^{-2}_{\v t} - I\right) \v u,
\end{align*}
with $\v u = \Pi^{-1}_{\v t}\v 1 \neq \v 0$. 
Since $\v t$ is in the unit interval, the matrix in the quadratic form is positive definite, which concludes the claim.

Finally, (iv) is a simple consequence of Lemma~\ref{lem:grad-hess} below which show that all the elements gradient $\nabla f_\delta(\v t)$ are non-positive. That means, in each coordinate $t_j$,  the objective function $f_\delta(\v t)$ is non-increasing along $t_j$. Thus, for each $k$, the simplex $\cS_k$ contains an optimal solution of Boolean relaxation~\eqref{eqn:opt-br2}.
\end{proof}

We now establish the gradient and the Hessian of the objective function $\wh f_\delta(\v t)$ using the scaling in Section~\ref{app:scaling}. Without scaling, the gradient and the Hessian of $f_\delta(\v t)$ are obtained by replacing $\v v$ with $\v 1$ in this lemma.

\begin{lemma}
    \label{lem:grad-hess}
    The gradient and the Hessian of $\wh f_\delta(\v t)$ at every $\v t \in (0,1)^p$ are respectively given by 
    \[
    \nabla \wh f_\delta(\v t) = -2\delta\frac{\v v \odot (\wh \Pi_{\v t}^{-1} \v 1)^2}{\v t^3},
    \]
and
    \[
    \wh H_\delta(\v t) = 2 \delta \diag\lt( \frac{\v v \odot (\wh \Pi_{\v t}^{-1} \v 1)}{\v t^3}\rt) \lt[3\diag(\v t^2 \odot \v w^2) - 4\delta  \wh{\Pi}_{\v t}^{-1} \rt] \diag\lt( \frac{\v v \odot (\wh \Pi_{\v t}^{-1} \v 1)}{\v t^3}\rt),
    \]
    where 
    \(
    \wh \Pi_{\v t} = \Sigma + \delta D_{\v t} \diag(\v v),
    \)
    and the division and power operations on vectors are element-wise.
\end{lemma}

\begin{proof}
Similar to Lemma~\ref{lem:simplified-f}, we can show that for any interior point $\v t \in (0, 1)^p$, 
\[
\wh f_\delta(\v t) = - \v 1^\top \wh \Pi_{\v t}^{-1} \v 1.
\]
Thus, for each $j =1, \dots, p$,
\begin{align*}
    \frac{\partial \wh f_\delta(\v t)}{\partial t_j} =  -\v 1^\top \frac{\partial \wh \Pi_{\v t}^{-1}}{\partial t_j} \v 1
    = \v 1^\top \wh \Pi_{\v t}^{-1} \frac{\partial \wh \Pi_{\v t}}{\partial t_j} \wh \Pi_{\v t}^{-1} \v 1.
\end{align*}
Since $\Sigma$ does not depend on $\v t$, from the definition of $\Pi_{\v t}$, 
\begin{align}
\frac{\partial \wh \Pi_{\v t}}{\partial t_j} = \delta \frac{\partial D_{\v t}}{\partial t_j} \diag(\v v) = - \frac{2 \delta v_j}{t_j^3} \v e_j \v e_j^\top,
\label{eqn:diff-Mt}
\end{align}
where $\v e_j$ is the vector with $1$ at the $j$-th location and $0$ everywhere else. 
Thus, we get the required expression for the gradient 
$\nabla \wh f_\delta(\v t)$. 

We now derive an expression of the Hessian of $\wh f(\v t)$. Toward this, to simplify the notation, let $\v u_{\v t} = \wh \Pi_{\v t}^{-1} \v 1$. Then, $j$-th column of the Hessian is given by
\begin{align*}
\frac{\partial }{\partial t_j} \nabla \wh f_\delta(\v t) &= \frac{6\delta v_j}{t_j^4} \v e_j \v e_j^\top (\v u_{\v t} \odot \v u_{\v t}) - 2\delta \lt(\frac{\v 1}{\v t^3} \rt) \odot \frac{\partial }{\partial t_j} (\v u_{\v t} \odot \v u_{\v t})\\
&= \frac{6\delta v_j}{t_j^4} \v e_j \v e_j^\top (\v u_{\v t} \odot \v u_{\v t}) - 4\delta \lt(\frac{\v u_{\v t}}{\v t^3}\rt) \odot \frac{\partial \v u_{\v t}}{\partial t_j} \\
&= \frac{6\delta v_j}{t_j^4} \v e_j \v e_j^\top (\v u_{\v t} \odot \v u_{\v t}) + 4\delta \lt(\frac{\v u_{\v t}}{\v t^3} \rt)\odot \left(\wh \Pi_{\v t}^{-1} \frac{\partial \wh \Pi_{\v t} }{\partial t_j} \v u_{\v t} \right).
\end{align*}
Using \eqref{eqn:diff-Mt}, we can write
\begin{align*}
\frac{\partial }{\partial t_j} \nabla \wh f_\delta(\v t)  
&= \frac{6\delta v_j}{t_j^4} \v e_j \v e_j^\top (\v u_{\v t} \odot \v u_{\v t}) -\frac{8\delta^2 v_j}{t_j^3} \lt(\frac{\v u_{\v t}}{\v t^3}\rt) \odot \left(\wh \Pi_{\v t}^{-1} \v e_j \v e_j^\top \v u_{\v t} \right).
\end{align*}
Consequently, the Hessian of $f_\delta$ at $\v t$ is given by
\begin{align*}
    \wh H_\delta(\v t) &= 6 \delta \lt[\diag\lt( \frac{\v v \odot \v u_{\v t}}{\v t^3}\rt)\rt]^2 T_{\v t}^2 \diag(\v 1/\v v) - 8\delta^2 \diag\lt( \frac{\v v \odot \v u_{\v t}}{\v t^3}\rt) \wh \Pi_{\v t}^{-1} \diag\lt( \frac{ \v v \odot \v u_{\v t}}{\v t^3}\rt)\nonumber\\
    &= 2 \delta \diag\lt( \frac{\v v \odot \v u_{\v t}}{\v t^3}\rt) \lt[3\diag(\v t^2 \odot \v w^2) - 4\delta  \wh \Pi_{\v t}^{-1} \rt] \diag\lt( \frac{\v v \odot \v u_{\v t}}{\v t^3}\rt),
\end{align*}
which completes the proof by recalling that $\v w = \v 1/ \sqrt{\v v}$.
\end{proof}

\begin{proof}[Proof of Theorem~\ref{thm:delta-prperties}] We establish the proof in generality under scaling, and Theorem~\ref{thm:delta-prperties} (no scaling) immediately holds by replacing $\v v$ with $\v 1$.

The Hessian $\wh H_\delta(\v t)$ is clearly symmetric at every $\v t \in (0, 1)^p$ and $\delta > 0$, but not necessarily positive definite. 
To show (i), note from Lemma~\ref{lem:grad-hess} that $\wh H_\delta(\v t)$ is negative definite if $4\delta  \wh \Pi_{\v t}^{-1} - 3\diag({\v t}^2 \odot \v w^2)$ is positive definite. To simplify the notation,  for a symmetric matrix $A$, we write $A \geq 0$ to denote that $A$ is positive semi-definite.

Suppose that $\Sigma = U \Delta U^\top$ is the singular value decomposition of $\Sigma$. Then, $\Sigma \leq \eta_1 I$ and thus, 
\begin{align*}
\wh \Pi_{\v t} = \Sigma + \delta D_{\v t} \diag(\v v)
\leq \eta_1 I + \delta D_{\v t} \diag(\v v)
= \diag\lt(\eta_1 \v 1 + \delta (\v 1/\v t^2 - \v 1) \odot \v v \rt).
\end{align*}
Hence, 
\[
\wh \Pi_{\v t}^{-1} \geq 
\diag \lt(\frac{\v t^2}{ \eta_1 \v t^2 + \delta (\v 1 - \v t^2) \odot \v v} \rt),
\]
and thus,
\begin{align*}
4\delta  \wh{\Pi}_{\v t}^{-1} - 3\diag(\v t^2 \odot \v w^2) \geq  \diag \lt(\frac{4\delta \v t^2}{ \eta_1 \v t^2 + \delta (\v 1 - \v t^2) \odot \v v} - 3 \frac{\v t^2}{\v v}\rt),
\end{align*}
here we used the fact that $\v v = 1/\v w^2$. Now it is enough to show that 
\begin{align*}
    \frac{4\delta}{\eta_1 t_j^2 + \delta (1 - t_j^2) v_j } - \frac{3}{v_j}  \geq 0 
\end{align*}
for all $j$ and $t_j \in (0, 1)$. Or equivalently, 
\begin{align*}
4\delta \geq \frac{3 }{v_j} \lt(\eta_1 t_j^2 + \delta (1 - t_j^2) v_j \rt).
\end{align*}
That is, 
\begin{align*}
(1 + 3t_j^2))\delta \geq 3 (\eta_1/v_j) t_j^2,
\end{align*}
or, 
\begin{align*}
\delta - 3 (\eta_1/v_j - \delta) t_j^2 \geq 0.
\end{align*}

Now, to establish the convexity property in (ii), observe that 
\begin{align*}
\wh \Pi_{\v t} = \Sigma + \delta D_{\v t} \diag(\v v)
\geq \eta_p I + \delta D_{\v t} \diag(\v v)
= \diag\lt(\eta_p \v 1 + \delta (\v 1/\v t^2 - \v 1) \odot \v v \rt).
\end{align*}
Hence, 
\[
\wh \Pi_{\v t}^{-1} \leq 
\diag \lt(\frac{\v t^2}{ \eta_p \v t^2 + \delta (\v 1 - \v t^2) \odot \v v} \rt),
\]
and thus,
\begin{align*}
4\delta  \wh{\Pi}_{\v t}^{-1} - 3\diag(\v t^2 \odot \v w^2) \leq  \diag \lt(\frac{4\delta \v t^2}{ \eta_p \v t^2 + \delta (\v 1 - \v t^2) \odot \v v} - 3 \frac{\v t^2}{\v v}\rt),
\end{align*}
Now it is enough to show that 
\begin{align*}
    \frac{4\delta}{\eta_p t_j^2 + \delta (1 - t_j^2) v_j } - \frac{3}{v_j}  \leq 0 
\end{align*}
for all $j$. Equivalently, 
\(
4\delta v_j \leq 3 \eta_p t_j^2 + 3\delta (1 - t_j^2) v_j.
\)
That is,
\(
\delta v_j \leq 3 \eta_p t_j^2 - 3\delta v_j t_j^2,
\)
or,
\begin{align*}
\delta  \leq \frac{1}{v_j}\lt(\frac{3 \eta_p t_j^2}{1 + 3 v_j t_j^2}\rt),
\end{align*}
which is true when for all $\v t \in (\varepsilon, 1)^p$ and $\delta > 0$ such that 
\[
\delta  \leq \frac{1}{\max_j v_j}\lt(\frac{3 \eta_p \varepsilon^2}{1 + 3  \varepsilon^2}\rt).
\]
\end{proof}

\begin{proof}[Proof of Theorem~\ref{thm:sol-continuity}]
Recall that for a fixed integer $k$,
    \[
    h_k(\delta) = \minimize_{\v t \in \cC_k} f_\delta(\v t)\quad 
    \text{and}
    \quad 
    \cD_{k,\delta} = \argmin_{\v t \in \cC_k} f_\delta(\v t).
    \]
Since $\Delta \coloneqq [\varepsilon, \eta_1]$ is compact for any $\varepsilon \in (0, \eta_1)$, from Chapter~9 of \cite{sundaram1996first}, $h_k(\delta)$ is continuous and $\cD_{k,\delta}$ is compact-valued, upper hemicontinuous\footnote{Upper hemicontinuity is sometimes referred to as upper semicontinuity.} correspondence on $\Delta$. This implies that for any sequence $\delta_\ell \to \eta_1$ and $\v t^{(\ell)} \in \cD_{k,\delta}$ with $\v t^{(\ell)} \to \v t^*$, we have $\v t^* \in \cD_{k,\eta_1}$.
We known from Theorem~\ref{thm:delta-prperties}(i) that  $f_{\eta_1}(\v t)$ is strictly concave on $\cC_k$, and thus, $\cD_{k,\eta_1}$ must be the solution set $\cD^*_k$ of the target optimization problem \eqref{eqn:opt-bc3}.
\end{proof}

\bibliographystyle{apalike}
\bibliography{Refs}

\end{document}